%% file: acl_latex.tex
\documentclass[11pt]{article}

\usepackage[final]{acl}

\usepackage{times}
\usepackage{latexsym}
\usepackage[T1]{fontenc}
\usepackage[utf8]{inputenc}
\usepackage{microtype}
\IfFileExists{inconsolata.sty}{\usepackage{inconsolata}}{}
\usepackage{graphicx}
\usepackage{enumitem}

\usepackage{amsmath}
\usepackage{amssymb}
\usepackage{amsthm}

\usepackage{booktabs}
\usepackage{multirow}
\usepackage{stfloats} 
\usepackage{float}    

\IfFileExists{algorithm.sty}{\usepackage{algorithm}}{%
  \newcommand{\algofallbackcaption}[1]{\par\noindent\textbf{Algorithm:} ##1\par\medskip}%
  \newenvironment{algorithm}[1][t]{%
    \par\medskip\noindent
    \let\caption\algofallbackcaption
  }{\par\medskip}%
}
\IfFileExists{algpseudocode.sty}{\usepackage{algpseudocode}}{%
  \newenvironment{algorithmic}[1][0]{\begin{quote}\ttfamily\small\obeylines}{\end{quote}}%
  \providecommand{\State}{}%
  \providecommand{\For}[1]{\textbf{for} ##1 \textbf{do}}%
  \providecommand{\EndFor}{\textbf{end for}}%
}

\usepackage{tikz}
\usetikzlibrary{positioning, calc}

\usepackage{xcolor}
\usepackage{tcolorbox}
\tcbuselibrary{breakable}

\newtheorem{proposition}{Proposition}

\title{Information-Gain Rewards over Diversity-Pruned Tests:\\GT-Anchored Verifier Co-Training for Reliable Code Generation}

\author{Ana Nunez \\
  Secure AI and Autonomy Lab \\
  University of Texas at San Antonio \\
  \texttt{ana.nunez@utsa.edu} \\\And
  Peyman Najafirad \\
  Secure AI and Autonomy Lab \\
  University of Texas at San Antonio \\
  \texttt{peyman.najafirad@utsa.edu} \\}

\begin{document}
\maketitle

\input{sections/00_abstract}

\input{sections/01_introduction}

\input{sections/02_related_work}

\input{sections/03_method}

\input{sections/04_experiments}

\input{sections/05_discussions}

\input{sections/06_limitations}

\bibliography{custom}

\appendix

\input{sections/appendix_b_additional_results}

\input{sections/appendix_c_experiment_details}

\input{sections/appendix_additional_results}

\input{sections/appendix_a_proofs}

\end{document}

%% file: sections/00_abstract.tex
\begin{abstract}
Self-play methods that co-train a single language model as both coder and test author promise to move code-generation RL beyond fixed test suites, but they suffer from two coupled pathologies: \emph{permissiveness collapse}, where pass-rate rewards are maximised by trivial, non-discriminative tests, and \emph{concentration bias}, where i.i.d.\@ sampled tests cluster on modal inputs and inflate estimator variance.
We introduce \textbf{CoVer} (\emph{Co-trained Coder and Verifier}), a single-policy GRPO framework that addresses both failure modes.
First, an \emph{information-gain (IG) reward} scores each self-generated test by the mutual information between its pass/fail vector and a graded, ground-truth-anchored correctness signal $y\!\in\![0,1]^m$, gated by the sign of their covariance so that only positively discriminative tests receive reward.
Second, a \emph{{three-stage diversity-aware selection}} step prunes a candidate pool to a behaviourally non-redundant suite (invalidity, input-string, execution-profile filtering), raising the effective sample size of the IG estimator at fixed execution budget.
On five benchmarks (LiveBench, MBPP, LiveCodeBench, CodeContests, CodeForces), CoVer raises one-shot pass@1 by \textbf{+5.8} points at 7B and \textbf{+7.1} points at 14B over the Qwen2.5-Instruct backbone, and achieves the highest macro-average among all compared methods at both scales. As a drop-in backbone inside the CodeT ranking pipeline, CoVer-7B adds \textbf{+3.5} points, demonstrating the dual benefit of co-training for both generation and selection.
\end{abstract}

%% file: sections/01_introduction.tex
\input{tables/table2_gpt_application}

\section{Introduction}
\label{sec:introduction}

\begin{figure}[t]
  \centering
  \includegraphics[width=\columnwidth]{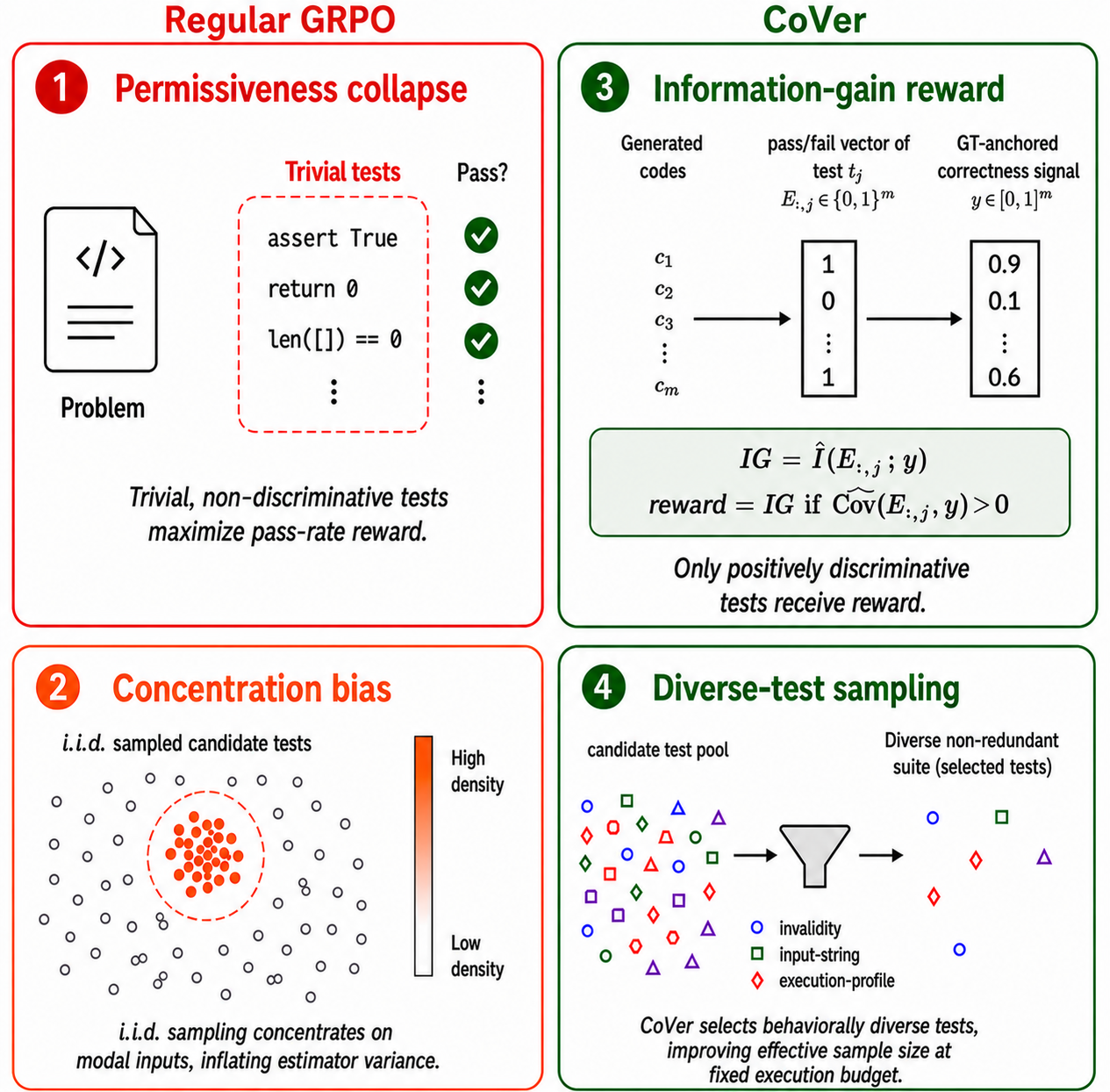}
  \caption{Two failure modes of standard self-play code RL (\textbf{left}) and CoVer's solutions (\textbf{right}). \textbf{(1)}~\emph{Permissiveness collapse}: trivial, non-discriminative tests maximise pass-rate and other binary-acceptance rewards while carrying zero information. \textbf{(2)}~\emph{Concentration bias}: i.i.d.\@ sampling clusters tests on modal inputs, inflating estimator variance. \textbf{(3)}~CoVer's IG reward scores each test by $I(E_{:,j};\,y)$ if $\mathrm{Cov}(E_{:,j}, y) > 0$, rewarding only positively discriminative tests. \textbf{(4)}~{Three-stage diversity-aware selection} (invalidity, input-string, execution-profile) selects a behaviourally diverse suite at fixed execution budget.}
  \label{fig:headline}
\end{figure}

Large language models (LLMs) have become the dominant tool for program synthesis, supporting tasks from competitive programming \citep{li2022competition,penedo2025codeforces} to functional code completion \citep{austin-mbpp}. The standard post-training recipe is Reinforcement Learning with Verifiable Rewards (RLVR): a policy is updated against an executable correctness signal, typically from a fixed unit-test suite. Group Relative Policy Optimization (GRPO) \citep{shao-deepseekmath}, together with PPO-based and critic-free RL-for-code methods \citep{dou-stepcoder,gehring-rlef,li-acecoder}, has made this loop a reliable recipe for improving one-shot \texttt{pass@1}. To move beyond fixed test sets, recent self-play work has the model author its own tests inside the RL loop and trains a single policy as both \emph{coder} and \emph{verifier} \citep{wang-cure,lin2025learning,wilf2025propose}, shifting correctness specification onto the model itself.

However, two coupled failure modes undermine existing self-play approaches. \textbf{Permissiveness collapse:} the verifier reward is a function of \emph{which} candidates a test accepts rather than \emph{how informative} the test is about correctness. A pass-rate reward, for instance, is maximised by a trivial constant-pass test that contributes zero discriminative signal; more broadly, rewards defined over binary acceptance patterns carry little information about a candidate's graded correctness, and the reward-maximisers are easy to generate, so the verifier drifts toward them. \textbf{Concentration bias:} tests drawn i.i.d.\@ from a single policy concentrate on a narrow region of modal inputs, producing highly correlated pass/fail columns in the execution table. Any reward computed from such a table inherits inflated variance and a biased view of program behaviour.

We propose \textsc{CoVer} (\emph{Co-trained Coder and Verifier}), a single-policy GRPO framework that addresses both pathologies through two coupled mechanisms (Figure~\ref{fig:pipeline}):
\begin{enumerate}[leftmargin=*,itemsep=2pt,topsep=2pt]
\item \textbf{Information-Gain (IG) Reward.} Each self-generated test is scored by the mutual information between its pass/fail column across $m$ candidate solutions and a \emph{graded} GT-anchored correctness signal $y\!\in\![0,1]^m$, gated by the sign of their covariance so that only positively discriminative tests receive reward. The graded form is essential: with a binary $y$, early in training almost no code passes every GT test, $H(y)\!\approx\!0$, and the IG reward vanishes. A graded $y$ preserves $H(y)$ from the first GRPO step and rewards tests that stratify codes by partial correctness.
\item \textbf{{Three-Stage Diversity-Aware Selection}.} We sample $K$ candidate tests per task and prune to a behaviourally non-redundant $k$-test suite through three discrete filters: (i)~removal of invalid tests, (ii)~input-string deduplication, and (iii)~execution-profile deduplication. The kept suite is a lower-variance estimator of the IG functional than direct $k$-sampling at matched execution budget, raising the effective sample size toward $k$.
\end{enumerate}

Both rewards are jointly optimized in a single-policy GRPO update, so gradients flow back to the same parameters under both roles. On five benchmarks \citep{white-livebench,austin-mbpp,jain-livecodebench,li2022competition,penedo2025codeforces}, \textsc{CoVer} achieves a macro-average one-shot \texttt{pass@1} of $35.94\%$ at the 7B scale and $42.90\%$ at 14B on the Qwen2.5-Instruct backbone \citep{yang2024qwen}, improving over the Instruct baseline and code-pretrained variants on every benchmark, and achieving the best macro-average among all compared methods at both scales.

\medskip
\noindent\textbf{Contributions.}
\begin{enumerate}[leftmargin=*,itemsep=1pt,topsep=2pt]
\item We propose \textsc{CoVer}, a single-policy GRPO framework that combines a covariance-gated MI verifier reward with {diversity-aware selection}, addressing permissiveness collapse and concentration bias in self-play code RL.
\item We frame self-generated test selection as a variance-reduction problem and prove that {diversity-aware selection} strictly reduces the variance of the IG estimator at matched execution budget.
\item We conduct experiments across five benchmarks at two scales, demonstrating consistent gains and isolating each design choice through targeted ablations.
\end{enumerate}

%% file: tables/table2_gpt_application.tex
\begin{table*}[t]
  \centering
  \small
  \setlength{\tabcolsep}{4pt}
  \resizebox{\textwidth}{!}{%
  \begin{tabular}{l ccccc}
    \toprule
    \multirow{2}{*}{\textbf{Method family}} &
      \textbf{GT only at} & \textbf{Trains a} & \textbf{MI-based} & \textbf{Diversity} & \textbf{No GT at} \\
      & \textbf{training} & \textbf{verifier} & \textbf{reward} & \textbf{selection} & \textbf{evaluation} \\
    \midrule
    Curated-test RL \citep{gehring-rlef,dou-stepcoder,li-acecoder}                   & \checkmark & --         & --         & --         & \checkmark \\
    Inference-time filtering \citep{chen-codet,ridnik-alphacodium}       & --         & --         & --         & --         & \checkmark \\
    Self-play co-training \citep{lin2025learning,wang-cure}              & --/\checkmark & \checkmark & --         & --         & \checkmark \\
    \textbf{CoVer (ours)}                                                 & \checkmark & \checkmark & \checkmark & \checkmark & \checkmark \\
    \bottomrule
  \end{tabular}%
  }
  \caption{Positioning CoVer against four threads of prior work. CoVer channels GT through a graded MI reward, prunes self-generated tests via diversity-aware selection for variance reduction, and evaluates in zero-shot mode with no GT signal.}
  \label{tab:gpt-application}
\end{table*}

%% file: sections/02_related_work.tex
\begin{figure*}[!t]
    \centering
    \includegraphics[width=\textwidth]{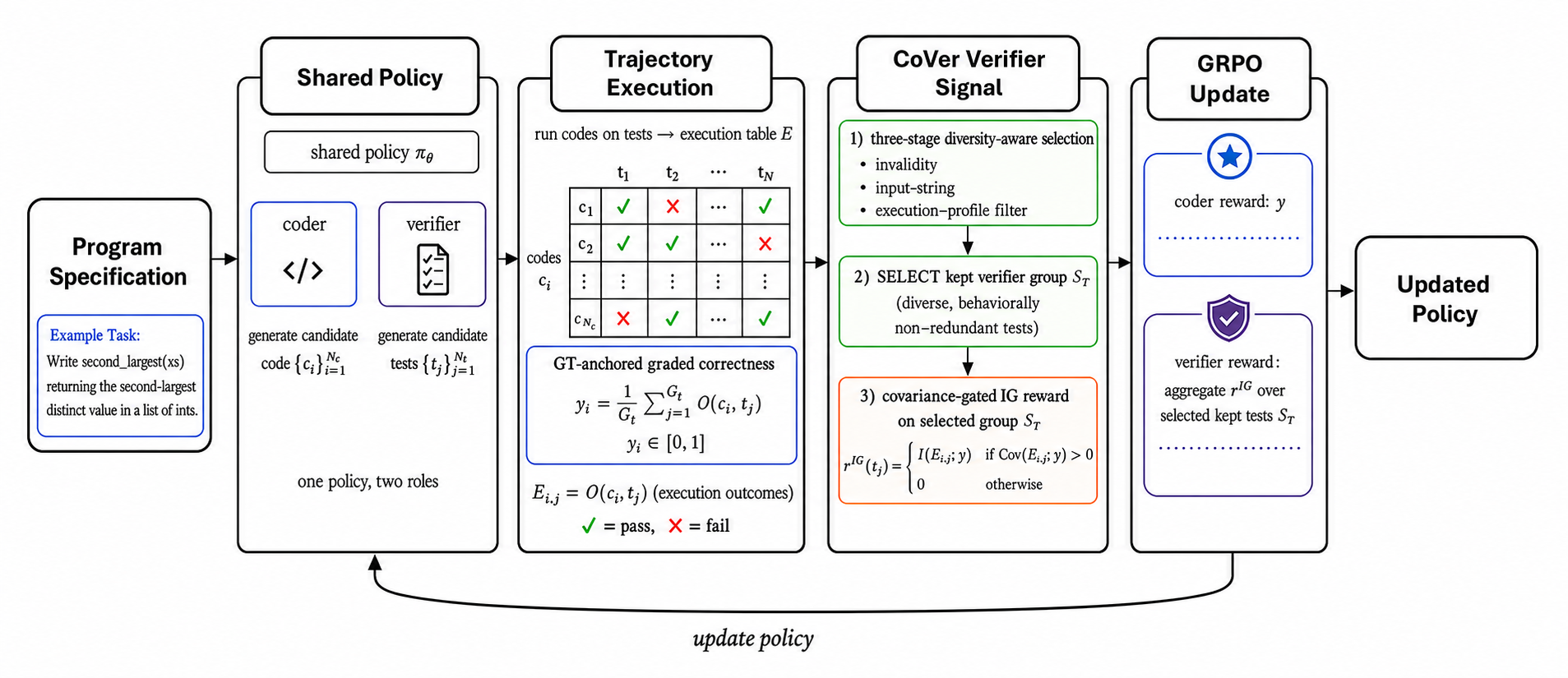}
    \caption{\textbf{Overview of the \textsc{CoVer} pipeline.} Given a \emph{program specification}, a single shared policy $\pi_\theta$ acts in two roles, \emph{coder} and \emph{verifier}, sampling a pool of candidate solutions $\{c_i\}$ and candidate tests $\{t_j\}$. During \emph{trajectory execution}, every solution is run against every test to populate an execution table $E$; in parallel, each solution is run against the ground-truth tests to compute a graded correctness signal $y$. The \emph{\textsc{CoVer} verifier signal} then (1) applies a three-stage diversity-aware selection to prune the pool to a behaviourally non-redundant suite, (2) keeps a verifier group $S_T$ of distinct tests, and (3) assigns each kept test a covariance-gated information-gain reward, so that only positively discriminative tests are rewarded. Finally, in the \emph{GRPO update}, the coder reward and the verifier IG reward over $S_T$ update the shared policy, which feeds back into the next round.}
\label{fig:pipeline}
\end{figure*}

\section{Related Work}
\label{sec:related}

\subsection{RL with Execution Feedback for Code}
\label{ssec:rel-curated}

PPO-based and related actor–critic methods \citep{schulman-ppo,le-coderl,shojaee-ppocoder} and GRPO \citep{shao-deepseekmath} reward the policy with pass/fail signals from dataset-shipped tests. StepCoder \citep{dou-stepcoder} decomposes rewards by execution sub-step; RLEF \citep{gehring-rlef} grounds the policy in multi-turn compiler and test feedback. ACECoder \citep{li-acecoder} synthesises tests offline and trains a separate Bradley--Terry reward model from pass-rate preferences. All these methods either use binary pass-aggregation as the per-step reward or train a separate reward model. CoVer uses the \emph{same} GT tests but channels them through a graded correctness signal that anchors a mutual-information verifier reward inside a single co-training loop, requiring no separate reward model.

\subsection{Inference-Time Test Generation and Filtering}
\label{ssec:rel-inference}

A complementary line generates tests at decoding time to filter, rerank, or cluster candidate solutions \citep{chen-codet,li2022competition,huang-mpsc,ridnik-alphacodium,li-sstar}. These methods are inference-only: the verification signal does not propagate into the policy. CoVer moves test-generation supervision into training so that a single backbone internalises both code and test distributions, and the resulting verifier transfers directly to inference-time scaling pipelines (\S\ref{sssec:inference-scaling}).

\subsection{Self-Generated Tests as RL Supervision}
\label{ssec:rel-selfplay}

CURE \citep{wang-cure} co-evolves coder and tester roles, rewarding each test with a precision estimator under a Bernoulli correctness model. Sol-Ver \citep{lin2025learning} shares our single-LLM two-role setup and sub-samples tests by branch coverage; supervision uses SFT followed by DPO rather than RL rewards. PSV \citep{wilf2025propose} relies on a formal verifier as oracle, a distinct supervision assumption. CoVer differs along two axes: (i)~the verifier reward is the signed MI between a test's pass/fail vector and a \emph{graded} GT-anchored signal $y$; and (ii)~tests are pruned by {three-stage diversity-aware selection} over execution profiles rather than by coverage on a single solution. Concurrent self-rewarding alignment \citep{yuan-selfrewarding,wu-metarewarding} and self-play critic work \citep{chen-spc} target non-code domains and are not directly comparable.

\subsection{Information-Theoretic Rewards}
\label{ssec:rel-mi}

Information-theoretic functionals have served as intrinsic rewards for exploration, variational information maximisation \citep{houthooft-vime}, and information-theoretic intrinsic motivation \citep{aubret-it-survey}. These rewards drive \emph{exploration}; CoVer repurposes the information-gain functional as the \emph{verifier} reward inside a single-policy GRPO loop, anchored by an external GT-derived correctness signal. 

Table~\ref{tab:gpt-application} positions CoVer against these four threads. Unlike prior self-play code RL methods that reward the verifier with pass-rate signals or discrimination scores and select tests by coverage heuristics, CoVer is, to our knowledge, the first to combine a covariance-gated MI verifier reward with execution-profile deduplication inside a single-policy GRPO loop, framing test selection as a variance-reduction problem.

%% file: sections/03_method.tex

\section{Method}
\label{sec:method}

\subsection{Overview}
\label{ssec:overview}

CoVer couples two mechanisms inside a single-policy GRPO loop. An \emph{information-gain verifier reward} (\S\ref{ssec:igreward}) scores each self-generated test by the mutual information between its pass/fail column and a graded GT-anchored correctness signal $y$, gated by the sign of their covariance. A \emph{\textcolor{black}{diversity-aware selection} step} (\S\ref{ssec:diversity}) prunes a candidate test pool to a behaviourally non-redundant suite, raising the effective sample size of the reward estimator at fixed execution cost. The coder is updated against $y$, the verifier against the IG reward, and both roles share a single policy (\S\ref{ssec:grpo}). Algorithm~\ref{alg:cover} summarises the procedure.

\begin{algorithm}[!t]
\caption{CoVer training loop.}
\label{alg:cover}
\begin{algorithmic}[1]
\State \textbf{Input:} task set $\mathcal{T}=\{\tau_1,\dots,\tau_N\}$ with GT tests $T^{\star}_\tau$; policy $\pi_\theta$; steps $M$; group sizes $m,K,k$ ($k<K$); learning rate $\eta$; KL coefficient $\beta$.
\For{$s=1$ to $M$}
  \For{each task $\tau\in\mathcal{T}$}
    \State Sample $m$ codes $\{c_i\}_{i=1}^{m}\sim\pi_\theta(\cdot\mid q_\tau,\text{coder})$.
    \State Sample $K$ tests $\{t_j\}_{j=1}^{K}\sim\pi_\theta(\cdot\mid q_\tau,\text{verifier})$.
    \State Compute graded $y\!\in\![0,1]^{m}$ via Eq.~\ref{eq:graded-y}.
    \State \textcolor{black}{Three-stage diversity-aware selection} (invalidity, input-string, execution-profile): rank by $b(\cdot)$ (Eq.~\ref{eq:dispersion}), keep $\mathcal{S}_\tau$ of size $k$.
    \State Build execution table $E\!\in\!\{0,1\}^{m\times k}$ for $(\{c_i\},\mathcal{S}_\tau)$.
    \State Compute $\mathcal{R}^{\mathrm{IG}}_{j}=r^{\mathrm{IG}}(t_j)$ for each $t_j\!\in\!\mathcal{S}_\tau$ (Eq.~\ref{eq:ig-reward}; zero unless $\widehat{\mathrm{Cov}}(E_{:,j},y)>0$).
    \State Coder rewards: $\mathcal{R}^{\mathrm{code}}_{i}=y_i$.
  \EndFor
  \State Compute group-normalised advantages $A^{\mathrm{code}}_{i}$, $A^{\mathrm{IG}}_{j}$ per task.
  \State $\theta\leftarrow\theta+\eta\nabla_\theta\mathcal{J}^{\mathrm{coder}}(\theta)$.
  \State $\theta\leftarrow\theta+\eta\nabla_\theta\mathcal{J}^{\mathrm{verifier}}(\theta)$.
\EndFor
\State \textbf{Output:} trained $\pi_\theta$.
\end{algorithmic}
\end{algorithm}

\subsection{Problem Formulation}
\label{sec:problem}

A coding task $\tau\!\sim\!\mathcal{T}$ consists of a problem statement $q_\tau$ and a GT test set $T^{\star}_\tau\!=\!\{u^{\star}_{\tau,g}\}_{g=1}^{G_\tau}$. A single policy $\pi_\theta$ generates sequences under two role-specific prompts: as \emph{coder} it samples $m$ candidate solutions $\{c_i\}_{i=1}^{m}\sim\pi_\theta(\cdot\mid q_\tau,\text{coder})$, and as \emph{verifier} it samples $K$ candidate tests $\{t_j\}_{j=1}^{K}\sim\pi_\theta(\cdot\mid q_\tau,\text{verifier})$. A deterministic sandbox $\mathcal{O}(c,t)\!\in\!\{0,1\}$ scores any code--test pair as pass~(1) or fail~(0). Running the sandbox yields two objects:

\paragraph{Graded GT-anchored correctness.}
\begin{equation}
\label{eq:graded-y}
y_i \;=\; \frac{1}{G_\tau}\sum_{g=1}^{G_\tau}\mathcal{O}(c_i,u^{\star}_{\tau,g})\;\in\;[0,1],
\end{equation}
the fraction of GT tests that $c_i$ passes, collected into $y\!=\!(y_1,\ldots,y_m)\!\in\![0,1]^m$.

\paragraph{Execution table.} Against a $k$-test subset $\mathcal{S}_\tau\!\subseteq\!\{t_j\}$ chosen by \textcolor{black}{the diversity-aware selection step} (\S\ref{ssec:diversity}), we record $E\!\in\!\{0,1\}^{m\times k}$ with $E_{ij}\!=\!\mathcal{O}(c_i,t_j)$. Each column $E_{:,j}$ is the pass/fail pattern of test $t_j$ across the $m$ codes and is the object the verifier is rewarded on.

The coder receives reward $\mathcal{R}^{\mathrm{code}}_i$ derived from $y$; the verifier receives $\mathcal{R}^{\mathrm{IG}}_j$ derived from $E$. Our goal is to maximise the coding objective $\mathbb{E}_{\pi_\theta(\cdot\mid q_\tau,\mathrm{coder})}[\mathcal{R}^{\mathrm{code}}]$, using the verifier role as training-time scaffolding whose GRPO updates flow back into the shared policy.

\subsection{Information-Gain Verifier Reward}
\label{ssec:igreward}

A test is informative when its pass/fail outcomes track codes' true correctness. A natural reward for each kept test $t_j\!\in\!\mathcal{S}_\tau$ is the mutual information between its column $E_{:,j}$ and the correctness vector $y$ from Equation~\ref{eq:graded-y}. MI alone, however, is unsigned: it credits any statistical dependence between $E_{:,j}$ and $y$, including adversarial patterns in which test passes track \emph{incorrect} codes. We propose a signed MI reward that zeros these out via a covariance-sign gate:
\begin{equation}
\label{eq:ig-reward}
r^{\mathrm{IG}}(t_j)\;=\;
\begin{cases}
\widehat{I}\bigl(E_{:,j}\,;\,y\bigr) & \text{if } \widehat{\mathrm{Cov}}(E_{:,j},y)>0,\\
0 & \text{otherwise,}
\end{cases}
\end{equation}
where $\widehat{\mathrm{Cov}}(E_{:,j}, y) = \tfrac{1}{m}\sum_{i}(E_{i,j}-\bar{E}_{\cdot,j})(y_i - \bar{y})$ is the sample covariance between the test's pass/fail vector and code quality: it is positive when above-average codes tend to pass $t_j$ and below-average codes tend to fail it, and negative when this relationship is inverted. The gate retains only tests whose pass/fail pattern moves in the same direction as correctness.

\paragraph{Estimator.} Both variables are discrete: $E_{i,j}\!\in\!\{0,1\}$ and $y_i\!\in\!\{0,\tfrac{1}{G_\tau},\ldots,1\}$. Their joint distribution fits a $2\!\times\!(G_\tau{+}1)$ contingency table populated by counting $(E_{i,j},y_i)$ outcomes. The plug-in estimator is
\begin{equation}
\label{eq:plug-in-mi}
\widehat{I}(E_{:,j};y) = \sum_{e,v}\widehat{P}(e,v)\log\frac{\widehat{P}(e,v)}{\widehat{P}(e)\,\widehat{P}(v)},
\end{equation}
where $e\!\in\!\{0,1\}$, $v\!\in\!\mathrm{supp}(y)$, and $0\log 0\!=\!0$. Because $y$ is discrete by construction, no binning or finite-sample bias correction is required.

\paragraph{Why graded $y$ is essential.} Binary $y\!\in\!\{0,1\}^{m}$ degenerates early in training: with almost no code passing all GT tests, $H(y)\!\approx\!0$ and $\widehat{I}(E_{:,j};y)\!\approx\!0$ for every test, leaving the verifier without signal in its most pliable phase. Graded $y$ keeps $H(y)$ nontrivial from the first GRPO step. A test distinguishing $6/7$-correct from $2/7$-correct codes contributes positive MI even when no code passes all GT tests.

\paragraph{What the IG reward zeros out.} Three test classes receive zero reward: (i)~\emph{constant tests} ($E_{:,j}\!\equiv\!0$ or $\equiv\!1$) yield $\hat{I}\!=\!0$; (ii)~\emph{independent tests} whose pass/fail vector is statistically independent of $y$; (iii)~\emph{anti-discriminative tests} with $\widehat{\mathrm{Cov}}(E_{:,j},y)\!\leq\!0$, suppressed by the sign gate. Positive reward is reserved for tests that cleanly stratify codes by correctness, directly penalising the permissiveness collapse inherent in pass-rate rewards.

\subsection{\textcolor{black}{Three-Stage Diversity-Aware Selection}}
\label{ssec:diversity}

A single policy sampled i.i.d.\@ produces tests that cluster around the most likely inputs, producing correlated execution columns. This correlation inflates the variance of per-test MI estimates and creates a reward amplification loop: redundant tests sharing the same pass/fail column receive correlated advantages, over-crediting the dominant test pattern and narrowing the verifier's output distribution \citep{yu-dapo,liu2025diversegrpo}. We address this with exact deduplication, a deliberately simple, zero-overhead mechanism whose effectiveness we justify empirically (\S\ref{sssec:test-selection}); soft similarity-based selection is an open direction.

We apply three discrete filters in sequence, following the principle that redundancy reduction improves estimator efficiency \citep{mirzasoleiman2020coresets}:

\begin{enumerate}[leftmargin=*,itemsep=1pt,topsep=2pt]
\item \textbf{Invalidity filter.} Remove tests that fail to parse or produce empty input/output.
\item \textbf{Input-string deduplication.} Remove tests sharing a whitespace-normalised input with another candidate, since under the deterministic oracle $\mathcal{O}$ identical inputs yield identical columns.
\item \textbf{Execution-profile deduplication.} Remove tests whose pass/fail signature across the $m$ codes exactly matches that of another retained test. Two such tests are functionally equivalent with respect to the current candidate set and contribute zero marginal information.
\end{enumerate}

We score each candidate test by the lexicographic key
\begin{equation}
\label{eq:dispersion}
  b(t_j) = \bigl\langle\, \mathbf{1}[t_j \text{ invalid}],\; d_{\mathrm{in}}(t_j),\; d_{\mathrm{col}}(t_j) \,\bigr\rangle,
\end{equation}
where $d_{\mathrm{in}}(t_j)$ counts other candidates sharing $t_j$'s normalised input and $d_{\mathrm{col}}(t_j)$ counts those sharing its execution profile. Sorting ascending by $b$ and retaining the top $k$ maximises the number of unique execution profiles. Invalid tests therefore rank last, entering the kept suite only when fewer than $k$ valid tests exist. The effective sample size of the kept suite is
\begin{equation}
  k_{\mathrm{eff}} \approx \frac{k}{1 + (k - 1)\,\bar{\rho}},
\end{equation}
where $\bar{\rho}$ is the average pairwise correlation of the kept columns. Diversity drives $\bar{\rho}$ toward zero and $k_{\mathrm{eff}}$ toward $k$, yielding a strictly lower-variance IG estimator than direct $k$-sampling at the same execution budget.

\subsection{Joint GRPO Update}
\label{ssec:grpo}

We optimize $\pi_\theta$ with GRPO \citep{shao-deepseekmath}:
\begin{equation}
\label{eq:rl-obj}
\begin{aligned}
&\mathcal{J}_{\mathrm{GRPO}}(\theta) =\\
&\mathbb{E}_{\{o^{(i)}\}\sim\pi_{\theta_{\mathrm{old}}}}\!\Biggl[\frac{1}{G}\sum_{i=1}^{G}\frac{1}{|o^{(i)}|}\sum_{t=1}^{|o^{(i)}|}\\
&\min\!\Bigl(\tfrac{\pi_\theta(o^{(i)}_t|x,o^{(i)}_{<t})}{\pi_{\theta_{\mathrm{old}}}(o^{(i)}_t|x,o^{(i)}_{<t})}A_i,\;
\mathrm{clip}(\cdot,\,1{-}\varepsilon,\,1{+}\varepsilon)A_i\Bigr)\Biggr].
\end{aligned}
\end{equation}
For the coder, $\mathcal{R}^{\mathrm{code}}_{i}\!=\!y_i$ (Eq.~\ref{eq:graded-y}); for the verifier, $\mathcal{R}^{\mathrm{IG}}_{j}\!=\!r^{\mathrm{IG}}(t_j)$ (Eq.~\ref{eq:ig-reward}), computed on the diversity-selected suite $\mathcal{S}_\tau$. Per-task rewards are mean-centred and standardised within each role's group; tasks with zero reward variance are dropped from that role's update.

%% file: sections/04_experiments.tex

\input{tables/table1_main_results}

\section{Experiments}
\label{sec:experiments}

\subsection{Experimental Settings}
\label{ssec:exp-settings}

\paragraph{Benchmarks.} We evaluate on five widely used coding benchmarks: LiveBench (128 tasks) \citep{white-livebench}, MBPP (221) \citep{austin-mbpp}, LiveCodeBench (v2, 511) \citep{jain-livecodebench}, CodeContests \citep{li2022competition}, and CodeForces (467) \citep{penedo2025codeforces}. Our CodeContests pool is restricted to difficulty $\leq\!2$ and split into 4.5k training and 239 evaluation tasks; the other four benchmarks are evaluation-only (Table~\ref{tab:data-splits}), and we confirmed the CodeForces set is disjoint from the training pool.

\paragraph{Training configuration.} We use Qwen2.5-7B-Instruct and Qwen2.5-14B-Instruct \citep{yang2024qwen} as base models, optimized with GRPO \citep{shao-deepseekmath}. At each step we draw $m\!=\!16$ candidate codes and $K\!=\!32$ candidate tests per task using vLLM \citep{kwon-vllm} with temperature~$1.0$ and top-$p$~$1.0$. Learning rate: $1\!\times\!10^{-6}$; KL coefficient $\beta\!=\!0.01$; clip $\varepsilon\!=\!0.2$; training: 350 steps on $2\!\times\!$B200 GPUs, with ablations trained under identical settings. The diversity-aware selection step retains $k\!=\!16$ tests per task.

\begin{figure}[h]
  \centering
  \includegraphics[width=0.7\columnwidth]{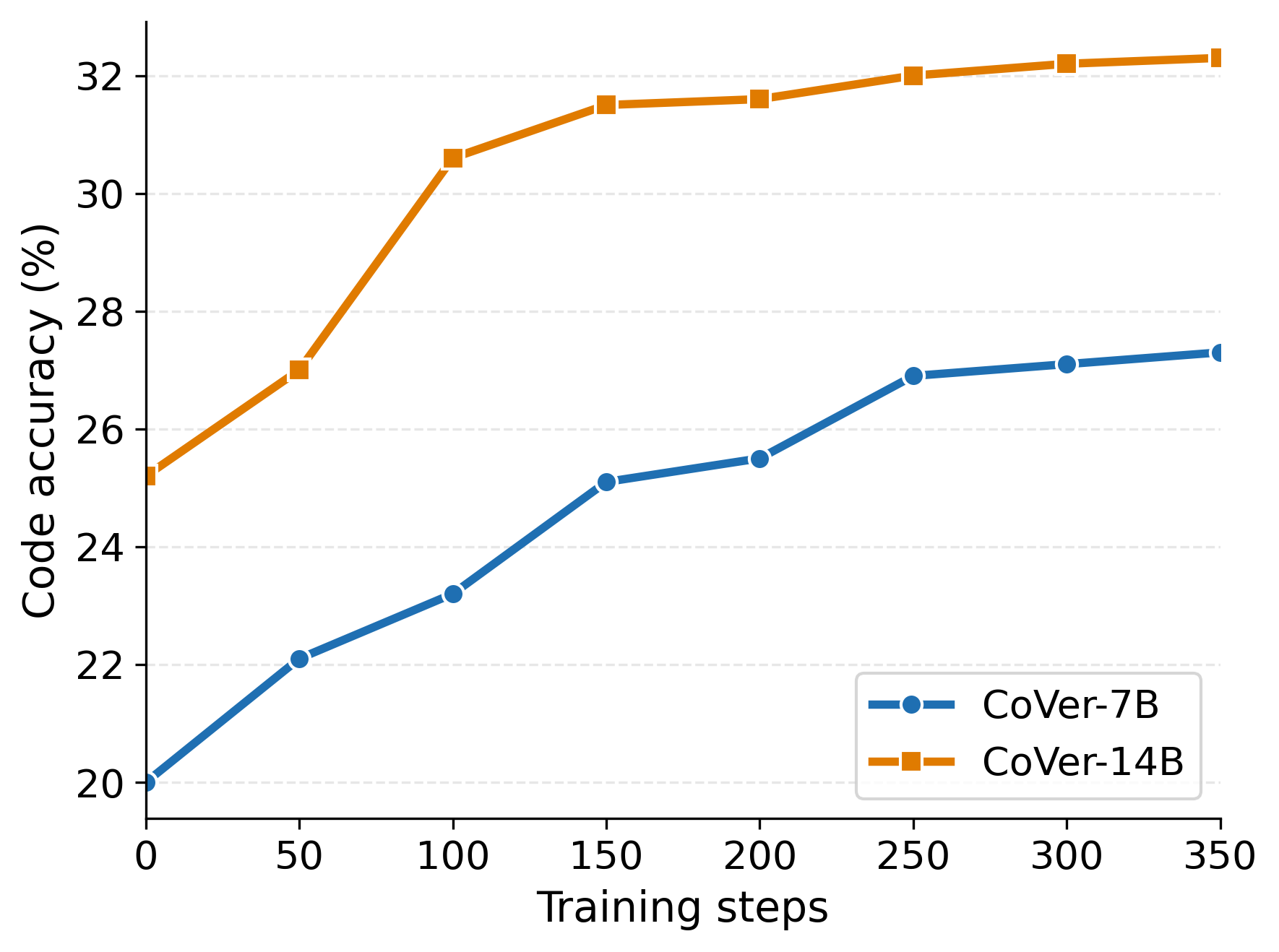}
  \caption{Code accuracy (pass@1, \%) versus training step for CoVer-7B and CoVer-14B. Both scales improve steadily over training, with the largest gains in the first $150$ steps.}
  \label{fig:accuracy-per-step}
\end{figure}

\paragraph{Baselines.} We compare against five families: (1)~\textbf{General-purpose backbones}: Qwen2.5-7B/14B-Instruct \citep{yang2024qwen}. (2)~\textbf{Code-pretrained LLMs}: Qwen2.5-7B/14B-Coder-Instruct \citep{hui-qwen25coder}. (3)~\textbf{RL with curated tests}: RLEF-8B \citep{gehring-rlef}, which fine-tunes on execution feedback from a fixed, human-curated test suite. (4)~\textbf{RL self-play methods}: Sol-Ver \citep{lin2025learning} (Llama 3.1 8B, coder--verifier co-training with DPO) and ReasonFlux-Coder-7B/14B \citep{wang-cure} (RL-based coder--verifier co-evolution). (5)~\textbf{Inference-time test filtering}: CodeT \citep{chen-codet} (ranks candidates by self-generated test pass count, no policy update).

\paragraph{Metric.} One-shot pass@1, reported per-benchmark and as macro-average across the five benchmarks.

\input{tables/table_ablations_dispersion}

\subsection{Results}
\label{ssec:results}

\subsubsection{Main Results}
\label{sssec:main}

Table~\ref{tab:main-results} reports one-shot pass@1 across five benchmarks. \textbf{CoVer achieves the best macro-average at both scales}: $35.94$ at 7B and $42.90$ at 14B, with both scales training stably and reaching these levels within $350$ steps (Figure~\ref{fig:accuracy-per-step}).

\input{tables/table_ablations_reward}
\label{sssec:test-selection}

\paragraph{7B scale.} CoVer improves over Qwen2.5-7B-Instruct by $+5.76$ macro points and surpasses the code-pretrained Qwen2.5-7B-Coder-Instruct ($32.46$) by $+3.48$ points, despite the latter's substantial code-pretraining advantage. CoVer achieves the highest macro-average among all 7B-scale methods, with the largest per-benchmark gains on LiveBench ($+8.9$) and LiveCodeBench ($+5.9$) over the Instruct backbone.

\paragraph{14B scale.} CoVer-14B is the top method on every benchmark, gaining $+11.35$ on LiveBench, $+3.14$ on MBPP, $+7.99$ on LiveCodeBench, $+7.32$ on CodeContests, and $+5.58$ on CodeForces over the Instruct backbone. It outperforms the strongest publicly reported 14B comparator, ReasonFlux-Coder-14B ($42.14$), on every benchmark.

\paragraph{} Table~\ref{tab:error-analysis} reports the per-benchmark pass@1 improvement ($\Delta$) of CoVer over both the instruct backbone and the coder-only baseline, together with the cross-seed standard deviation. All numbers are averaged over $4$ random seeds. For cross-training-run variance, we additionally trained a second, fully independent run at each scale; the runs differ by $0.61$ points at 7B and $0.35$ at 14B in macro-average pass@1, small relative to the $+5.76$ and $+7.08$ gains over the backbones (Appendix~\ref{ssec:appendix-seed-variance}).

\paragraph{Against RL baselines.} CoVer's margins are positive on every benchmark each RL baseline reports: $+16.3$\,pp on CodeContests (RLEF-8B), $+31.4$/$+5.6$\,pp on MBPP/LiveCodeBench (Sol-Ver), and positive on every benchmark at both scales against ReasonFlux-Coder. 




\subsubsection{Diversity-aware Selection Reduces Estimator Variance}

Table~\ref{tab:abl-dispersion} evaluates the effect of diversity-aware selection by comparing CoVer against a no-pruning baseline at \emph{matched sampling budget} on Qwen2.5-7B-Instruct. Both variants draw $K\!=\!32$ tests and execute all $m\!\times\!K$ code--test pairs; CoVer prunes to a diverse $k\!=\!16$ subset for the IG reward, while the baseline uses all $K\!=\!32$ columns. The comparison therefore isolates the value of column-level deduplication in the IG estimator.

Diversity-aware selection improves the macro-average pass@1 by $1.67$ points ($34.27 \rightarrow 35.94$), with positive gains on every benchmark. The consistent gain reflects how much signal the verifier loses when tests sampled i.i.d.\ from the policy produce correlated pass/fail columns. Using \emph{more} i.i.d.-sampled columns in the IG estimate does not compensate for their redundancy, because correlated columns contribute little to the effective sample size of the per-test MI estimate. Pruning recovers the lost signal at zero additional execution cost.

\subsubsection{Ablation Study}
\label{sssec:ablations}




Table~\ref{tab:abl-reward} ablates CoVer's three reward-design components on Qwen2.5-7B-Instruct, keeping all other settings fixed. We evaluate three variants:

\begin{itemize}[leftmargin=*,itemsep=1pt,topsep=2pt]
\item \textbf{A1: Pass-rate verifier reward.} Replaces the IG reward with a simplified reward for the unit test that assigns 1 if all correct codes pass it and 0 otherwise $r^{\mathrm{pass}}(t_j)\!=\!\mathbf{1}[\text{all correct codes pass } t_j]$.
\item \textbf{A2: Binary $y$.} Replaces the graded correctness vector $y$ with a binary all-or-nothing $y^{\mathrm{bin}}_i\!=\!\prod_g\mathcal{O}(c_i,u^{\star}_g)$.
\item \textbf{A3: Coder only.} Drops the verifier role entirely; trains on $R^{\mathrm{coder}}\!=\!y_i$ alone.
\end{itemize}

A1 drops the macro-average by $2.03$ points ($35.94 \rightarrow 33.91$), confirming that mutual information is a better-shaped signal than pass-rate for self-generated tests. A2 drops the macro-average by $1.19$ points ($35.94 \rightarrow 34.75$); the smaller gap relative to A1 is consistent with the entropy-bound argument of Section~\ref{ssec:igreward}, since partial correctness still provides usable signal once any code passes any GT test. A3 gives the largest degradation, $2.43$ points ($35.94 \rightarrow 33.51$), showing that the verifier role supplies gradient information that the coder loss alone cannot recover.

\subsubsection{Application to Test-Time Scaling}
\label{sssec:inference-scaling}

\input{tables/table_inference_scaling}
\input{tables/table3_error_analysis}

We evaluate CoVer as the generator inside CodeT \citep{chen-codet}, which samples multiple candidates and ranks them by the number of self-generated tests they pass. CodeT's effectiveness therefore depends on both the quality of the candidate codes and the discriminative power of the self-generated tests that score them, both of which are capabilities CoVer is designed to improve. Replacing the Qwen2.5-7B-Instruct code generator with CoVer-7B inside the CodeT pipeline raises the macro-average pass@1 from $38.97$ to $42.47$, a gain of $3.50$ points over the CodeT baseline and $12.29$ points over the bare Instruct backbone (Table~\ref{tab:inference-scaling}).



%% file: tables/table1_main_results.tex
\begin{table*}[t]
  \centering
  \small
  \setlength{\tabcolsep}{5pt}
  \begin{tabular}{l ccccc c}
    \toprule
    \textbf{Model} &
      \textbf{LiveBench} & \textbf{MBPP} & \textbf{LCB} &
      \textbf{CC} & \textbf{CF} & \textbf{Avg.} \\
    \midrule
    \multicolumn{7}{l}{\textit{Baselines ($\sim$7--8B scale)}}\\
    \midrule
    RLEF-8B (Llama 3.1 8B Inst) \citep{gehring-rlef}                   & -- & -- & -- & 9.7 & -- & -- \\
    Sol-Ver (Llama 3.1 8B) \citep{lin2025learning}  & -- & 41.0 & 27.24 & -- & -- & -- \\
    ReasonFlux-Coder-7B \citep{wang-cure}           & 37.1 & 70.2 & 31.2 & 25.9 & 8.2 & 34.52 \\
    Qwen2.5-7B-Coder-Instruct                       & 35.0 & 68.0 & 29.8 & 22.8 & 6.7 & 32.46 \\
    Qwen2.5-7B-Instruct                             & 31.1 & 66.3 & 26.9 & 21.2 & 5.4 & 30.18 \\
    \midrule
    \textbf{CoVer-7B (ours)}                         & \textbf{40.0} & \textbf{72.4} & \textbf{32.8} & \textbf{26.0} & \textbf{8.5} & \textbf{35.94} \\
    \midrule
    \multicolumn{7}{l}{\textit{Baselines ($\sim$14B scale)}}\\
    \midrule
    ReasonFlux-Coder-14B \citep{wang-cure}          & 47.5 & 78.5 & 40.5 & 32.1 & 12.1 & 42.14 \\
    Qwen2.5-14B-Coder-Instruct                      & 42.2 & 72.6 & 38.2 & 23.3 & 7.8 & 36.82 \\
    Qwen2.5-14B-Instruct                            & 36.4 & 76.3 & 33.5 & 25.6 & 7.3 & 35.82 \\
    \midrule
    \textbf{CoVer-14B (ours)}                        & \textbf{47.75} & \textbf{79.44} & \textbf{41.49} & \textbf{32.92} & \textbf{12.88} & \textbf{42.90} \\
    \bottomrule
  \end{tabular}
  \caption{CoVer achieves the highest macro-average one-shot pass@1 (\%) at both 7B and 14B scales. We compare against four families of baselines: general-purpose backbones (Qwen2.5-Instruct), code-pretrained LLMs (Qwen2.5-Coder-Instruct), Curated-test RL (RLEF), and RL self-play methods for code (Sol-Ver, ReasonFlux-Coder). LCB = LiveCodeBench, CC = CodeContests, CF = CodeForces. Backbone is noted in parentheses where it differs from Qwen2.5.}
  \label{tab:main-results}
\end{table*}

%% file: tables/table_ablations_dispersion.tex
\begin{table*}[!t]
  \centering
  \small
  \setlength{\tabcolsep}{5pt}
  \begin{tabular}{l ccccc c}
    \toprule
    \textbf{Variant (Qwen2.5-7B-Instruct)} &
      \textbf{LiveBench} & \textbf{MBPP} & \textbf{LCB} &
      \textbf{CC} & \textbf{CF} & \textbf{Avg.} \\
    \midrule
    \textbf{CoVer-7B (diversity selection, $k\!=\!16$)}  & \textbf{40.0} & \textbf{72.4} & \textbf{32.8} & \textbf{26.0} & \textbf{8.5} & \textbf{35.94} \\
    No pruning (all $K\!=\!32$ columns)               & 37.64 & 70.0 & 30.68 & 24.81 & 8.2 & 34.27 \\
    \midrule
    $\Delta$                                          & +2.36 & +2.4 & +2.12 & +1.19 & +0.3 & +1.67 \\
    \bottomrule
  \end{tabular}
  \caption{Three-stage diversity-aware selection ablation at \emph{matched sampling budget} ($K\!=\!32$, $m\!=\!16$). Both variants execute $m\!\times\!K$ code--test pairs; CoVer prunes to a diverse $k\!=\!16$ columns for the IG reward. The consistent gain confirms that removing redundant columns recovers verifier signal lost to correlated test outcomes.}
  \label{tab:abl-dispersion}
\end{table*}

%% file: tables/table_ablations_reward.tex
\begin{table}[!b]
  \centering
  \footnotesize
  \setlength{\tabcolsep}{3pt}
  \begin{tabular}{@{}l cccc c c@{}}
    \toprule
    \textbf{Variant} &
      \textbf{LB} & \textbf{MBPP} & \textbf{LCB} &
      \textbf{CC} & \textbf{CF} & \textbf{Avg.} \\
    \midrule
    \textbf{CoVer-7B} & \textbf{40.0} & \textbf{72.4} & \textbf{32.8} & \textbf{26.0} & \textbf{8.5} & \textbf{35.94} \\
    A1: pass-rate rew. & 36.32 & 69.25 & 31.25 & 24.94 & 7.8 & 33.91 \\
    A2: binary $y$ & 37.98 & 70.5 & 31.73 & 25.52 & 8.0 & 34.75 \\
    A3: coder only & 38.03 & 67.7 & 30.19 & 23.92 & 7.7 & 33.51 \\
    \bottomrule
  \end{tabular}
  \caption{Reward-design ablation on Qwen2.5-7B-Instruct.
  A1: pass-rate reward replaces IG; A2: binary $y$ replaces graded $y$; A3: no verifier role. Each ablation removes one design choice; all three degrade performance. All training runs are conducted with 350 optimization steps.}
  \label{tab:abl-reward}
\end{table}

%% file: tables/table_inference_scaling.tex
\begin{table*}[t]
  \centering
  \small
  \setlength{\tabcolsep}{5pt}
  \begin{tabular}{l cccccc}
    \toprule
    \textbf{Method} &
      \textbf{LiveBench} & \textbf{MBPP} & \textbf{LCB} &
      \textbf{CC} & \textbf{CF} & \textbf{Avg.} \\
    \midrule
    Qwen2.5-7B-Instruct (one-shot) & 31.1 & 66.3 & 26.9 & 21.2 & 5.4 & 30.18 \\
    CodeT ($n\!=\!8$)               & 44.53 & 75.95 & 37.76 & 28.03 & 8.56 & 38.97 \\
    \textbf{CodeT + CoVer-7B ($n\!=\!8$)}  & \textbf{50.78} & \textbf{79.63} & \textbf{40.70} & \textbf{29.70} & \textbf{11.56} & \textbf{42.47} \\
    \quad $\Delta$ vs CodeT         & \textit{+6.25} & \textit{+3.68} & \textit{+2.94} & \textit{+1.67} & \textit{+3.00} & \textit{+3.50} \\
    \bottomrule
  \end{tabular}
  \caption{Test-time scaling with CodeT \citep{chen-codet} ($n\!=\!8$ candidates). Replacing the backbone with CoVer-7B raises the macro-average by $+3.50$ points, confirming that co-training produces both a stronger coder and a more discriminative verifier.}
  \label{tab:inference-scaling}
\end{table*}

%% file: tables/table3_error_analysis.tex
\begin{table*}[t]
  \centering
  \small
  \setlength{\tabcolsep}{5pt}
  \begin{tabular}{l ccccc}
    \toprule
    \multirow{2}{*}{\textbf{Comparison}} &
      \textbf{LiveBench} & \textbf{MBPP} & \textbf{LCB} &
      \textbf{CC} & \textbf{CF} \\
      & $\Delta$ (std) & $\Delta$ (std) & $\Delta$ (std) & $\Delta$ (std) & $\Delta$ (std) \\
    \midrule
    CoVer-7B vs Qwen2.5-7B-Inst  & +8.9\,(0.36) & +6.1\,(0.47) & +5.9\,(0.31) & +4.8\,(0.33) & +3.1\,(0.28) \\
    CoVer-14B vs Qwen2.5-14B-Inst & +11.4\,(0.14) & +3.1\,(0.32) & +8.0\,(0.27) & +7.3\,(0.37) & +5.6\,(0.20) \\
    \midrule
    CoVer-7B vs Coder-7B          & +5.0\,(0.36) & +4.4\,(0.47) & +3.0\,(0.31) & +3.2\,(0.33) & +1.8\,(0.28) \\
    CoVer-14B vs Coder-14B        & +5.6\,(0.14) & +6.8\,(0.32) & +3.3\,(0.27) & +9.6\,(0.37) & +5.1\,(0.20) \\
    \bottomrule
  \end{tabular}
  \caption{Per-benchmark pass@1 improvement ($\Delta$) with cross-seed standard deviation (in parentheses) over $4$ seeds. CoVer shows consistent improvements across all benchmarks, with gains exceeding the cross-seed variance on every comparison.}
  \label{tab:error-analysis}
\end{table*}

%% file: sections/05_discussions.tex

\section{Conclusion}
\label{sec:discussions}

We introduced CoVer, a single-policy GRPO framework that co-trains an LLM as both coder and verifier, addressing two pathologies of prior self-play methods: permissiveness collapse from pass-rate rewards and concentration bias from i.i.d.\@ test sampling. CoVer replaces these with a covariance-gated MI verifier reward anchored to a graded GT-derived correctness signal and a three-stage diversity-aware selection step that prunes redundant execution columns to raise the effective sample size at fixed cost. On five coding benchmarks, CoVer achieves the best macro-average pass@1 at both 7B and 14B scales, surpassing code-pretrained variants while training from a general-purpose Instruct backbone. As a drop-in backbone for the CodeT reranking pipeline, CoVer-7B adds $+3.50$ points, confirming that co-training yields both a stronger coder and a more discriminative verifier.

This work demonstrates that framing verifier learning as an information-theoretic, variance-aware estimation problem yields a principled and empirically stronger path than pass-rate maximisation. The underlying principle is to reward a verifier for the information its probes add about a fixed correctness anchor, measured over a low-redundancy probe set. This recipe may extend to domains beyond unit-test generation. We hypothesise that mathematical reasoning (where partial proofs serve as probes) and structured prediction (where constraint-satisfaction checks serve as probes) are natural next steps, provided a graded correctness signal can be defined.

%% file: sections/06_limitations.tex
\section*{Limitations}
\label{sec:limitations}

Several factors constrain the scope of our conclusions. \textbf{Model family coverage.} Our experiments use the Qwen2.5-Instruct family at 7B and 14B, as well as Llama-3.1-Instruct at 8B (Appendix~\ref{ssec:appendix-cross-family}); transfer to additional families (e.g., DeepSeek-V2, Gemma) is not yet established. We initialise from the general-purpose Instruct variant rather than the code-pretrained Coder variant by design, so that gains can be attributed to CoVer rather than to code pretraining. \textbf{GT test dependency.} CoVer requires each training task to ship with at least one ground-truth test for computing the graded $y$; a fully bootstrapped variant that estimates $y$ from the execution table (e.g., via majority vote or EM) is left to future work, as it must address permissiveness collapse and coordinated hallucinations across the coder and verifier roles. \textbf{Exact deduplication only.} The diversity-aware selection step removes exact input-string and column duplicates but does not penalise high-correlation near-duplicates; soft column-similarity criteria (e.g., cosine-similarity clustering over execution-profile embeddings) may recover additional signal and are an open direction. \textbf{Determinism assumption.} CoVer presupposes deterministic sandbox execution of generated tests; tasks involving non-deterministic side effects, GUI interaction, or distributed systems fall outside its scope. \textbf{Distribution shift.} The verifier is trained on the policy's own code distribution; out-of-distribution code at deployment may be misclassified.

\section*{Ethical considerations}
\label{sec:ethics}

CoVer improves the reliability of code-generation models by training a single policy as both coder and verifier under GT supervision. Improvements in pass@1 do not constitute correctness guarantees; developers should independently verify generated code, especially in safety-critical contexts. Like any more capable code generator, CoVer marginally lowers the cost of malware authoring; we follow ACL's Publication Ethics policy regarding responsible release. All generated tests are executed inside a sandboxed Python subprocess with restricted system-call access, capped wall-clock time (5\,s per execution), and capped memory (512\,MB). GPU-hours and energy footprint are reported in Appendix~\ref{sec:appendix-c}. The benchmark datasets (LiveBench, MBPP, LiveCodeBench, CodeContests, CodeForces) are publicly released and used under their respective licences.

%% file: sections/appendix_b_additional_results.tex
\section{Prompts used by Coder and Tester}
\label{sec:appendix-b}

\begin{tcolorbox}[breakable,colback=gray!5,colframe=gray!50,title={\textbf{Coder Prompt}}]
    \begin{verbatim}
    <|im_start|>You are a helpful
    assistant that helps users solve
    programming problems. <|im_end|>
    <|im_start|>User: Think through
    the problem before writing your
    Python solution. Use input() for
    reading input and print() for
    producing output in your script.
    Ensure your output matches the
    required format exactly, with no
    extra spaces or lines.
    Here is the problem:
    <problem_statement> <|im_end|>
    <|im_start|>Assistant:
    \end{verbatim}
    \end{tcolorbox}
    
    \begin{tcolorbox}[breakable,colback=gray!5,colframe=gray!50,title={\textbf{Tester Prompt}}]
    \begin{verbatim}
    <|im_start|>You are a helpful
    assistant specialized in
    generating test examples for
    coding tasks. <|im_end|>
    <|im_start|>User: Given a coding
    task, your goal is not to write
    the solution, but to generate a
    new test example consisting of
    an input, an expected output,
    and an explanation.
    Here is the problem:
    <problem_statement>

    Your test example must be
    completely accurate and conform
    to the problem's format
    requirements, while also being
    discriminative enough to
    distinguish correct code from
    incorrect code.
    Begin by thinking carefully and
    reasoning step by step inside
    <reasoning> tags to derive an
    input and output you are
    confident are correct. A good
    approach is to first design an
    input you can reliably work
    through, then compute the output
    step by step. If you are unsure
    about the output, revise or
    redesign the input until you are
    certain. Skipping this process
    and directly providing
    input/output pairs is strongly
    discouraged, as it frequently
    leads to incorrect results.
    Once you have thoroughly
    completed your reasoning and
    derivation, your final response
    MUST follow this exact format:
    
    <reasoning>
    your explanation here.
    </reasoning>
    
    <answer>
    <input>
    your raw input here
    </input>
    <output>
    your raw output here
    </output>
    </answer>
    
     <|im_end|>
    <|im_start|>Assistant: <reasoning>
    \end{verbatim}
    \end{tcolorbox}

%% file: sections/appendix_c_experiment_details.tex
\section{Experiment Details}
\label{sec:appendix-c}

\subsection{Data Preprocessing}
\label{ssec:appendix-c2}

\paragraph{Benchmark coverage.} All five benchmarks consist of English-language natural-language problem statements paired with executable test cases, and all solutions are evaluated as Python~3.10 programs (\S\ref{ssec:appendix-c4}). They span two domains: \emph{competitive programming}, covered by CodeContests~\citep{li2022competition}, CodeForces~\citep{penedo2025codeforces}, and LiveCodeBench~\citep{jain-livecodebench}, where problems involve algorithmic reasoning over stdio I/O; and \emph{functional code completion}, covered by MBPP~\citep{austin-mbpp} and the function-level subset of LiveBench~\citep{white-livebench}, where problems specify a target function from a short description.

\paragraph{Format.} We adopt the stdio input/output format used in LiveBench~\citep{white-livebench}, LiveCodeBench~\citep{jain-livecodebench}, CodeContests~\citep{li2022competition}, and CodeForces~\citep{penedo2025codeforces}. For MBPP~\citep{austin-mbpp} and a subset of LiveBench/LiveCodeBench tasks that are released in a functional format, we use the pre-converted stdio versions provided by \citet{wang-cure}, in which each variable is placed on a separate line and lists are flattened into space-separated values on a single line.

\begin{table}[h]
    \centering
    \small
    \begin{tabular}{lrr}
    \toprule
    \textbf{Benchmark} & \textbf{Train} & \textbf{Eval} \\
    \midrule
    CodeContests (diff.\ $\leq\!2$) & 4{,}500 & 239 \\
    LiveBench            & -- & 128 \\
    MBPP                 & -- & 221 \\
    LiveCodeBench (v2)   & -- & 511 \\
    CodeForces           & -- & 467 \\
    \midrule
    Total                & 4{,}500 & 1{,}527 \\
    \bottomrule
    \end{tabular}
    \caption{Training and evaluation task counts. RL training uses only the CodeContests training partition; the other four benchmarks contribute no training data and are used solely for evaluation.}
    \label{tab:data-splits}
\end{table}

\paragraph{Licensing.} All artifacts are used in accordance with their licenses and intended research use. The Qwen2.5-7B/14B-Instruct backbones~\citep{yang2024qwen} are released under Apache~2.0, and Llama-3.1-8B-Instruct~\citep{grattafiori2024llama3}, used for the cross-family experiment in Appendix~\ref{ssec:appendix-cross-family}, under the Llama~3.1 Community License. Among the evaluation benchmarks, MBPP~\citep{austin-mbpp} is distributed under CC-BY-4.0, CodeContests~\citep{li2022competition} and LiveBench~\citep{white-livebench} under Apache~2.0, LiveCodeBench~\citep{jain-livecodebench} under the MIT license, and CodeForces~\citep{penedo2025codeforces} under CC-BY-4.0. Our training pipeline uses vLLM~\citep{kwon-vllm} and OpenRLHF (both Apache~2.0) and PyTorch (BSD-3).

\subsection{Hyperparameters}
\label{ssec:appendix-c3}

Table~\ref{tab:hyperparams} lists the full hyperparameter configuration.

\begin{table}[h]
\centering
\small
\begin{tabular}{ll}
\toprule
\textbf{Parameter} & \textbf{Value} \\
\midrule
Backbones & Qwen2.5-7B/14B-Instruct \\
 & Llama-3.1-8B-Instruct \\
Candidate codes ($m$) & 16 \\
Sampled tests ($K$) & 32 \\
Kept tests ($k$) & 16 \\
Learning rate & $1 \times 10^{-6}$ \\
KL coefficient ($\beta$) & 0.01 \\
Clip $\varepsilon$ & 0.2 \\
Training steps & 350 \\
Temperature (train) & 1.0 \\
Top-$p$ (train) & 1.0 \\
Temperature (eval) & 0.0 (greedy) \\
Hardware & $2 \times$ NVIDIA B200 \\
\midrule
\multicolumn{2}{l}{\textbf{Three-stage diversity-aware selection}} \\
Invalidity filter & parse failure or empty I/O \\
Input dedup & whitespace-normalised exact \\
Column dedup & exact pass/fail signature \\
Ordering & lexicographic on Eq.~\ref{eq:dispersion} \\
\midrule
\multicolumn{2}{l}{\textbf{IG estimator}} \\
Estimator & plug-in MI, $2 \times (G_\tau{+}1)$ table \\
Bias correction & none (finite discrete support) \\
\bottomrule
\end{tabular}
\caption{CoVer hyperparameter configuration.}
\label{tab:hyperparams}
\end{table}

\subsection{Sandbox Configuration}
\label{ssec:appendix-c4}

All execution of candidate code and self-generated tests takes place inside a sandboxed Python subprocess. Each execution is constrained by: (i)~restricted system-call access (no network, no filesystem writes outside a temporary directory), (ii)~a 5-second wall-clock timeout per code--test pair, and (iii)~a 512\,MB memory cap. Tests that trigger timeout or memory errors are treated as failures ($\mathcal{O}\!=\!0$). The sandbox uses Python 3.10 with \texttt{resource} module limits on Linux.

\subsection{CodeT Baseline}
\label{ssec:appendix-c-codet}

CodeT~\citep{chen-codet} is an inference-time test-filtering method that requires no policy update. Given a problem, it samples $n$ candidate solutions and a set of self-generated tests from the same model, then executes every candidate against every test. Candidates are grouped by their pass/fail signatures, and each candidate is scored by the number of self-generated tests it passes (weighted by the size of its agreement cluster); the highest-scoring candidate is returned. CodeT thus relies on both the quality of the sampled candidates and the discriminative power of the self-generated tests, the two capabilities CoVer is designed to improve.

\subsection{Ablation Definitions}
\label{ssec:appendix-c5}

The ablations referred to in Section~\ref{sec:experiments} are operationalised as follows.

\textbf{Direct-$k$ at matched test-execution budget.} Skip the diversity-aware selection step. Sample $k$ tests directly from $\pi_\theta(\cdot\mid q_\tau,\text{verifier})$ and use them as $\mathcal{S}_\tau$.

\textbf{(A1) Simple Pass-rate verifier reward.} Replace $r^{\mathrm{IG}}(t_j)$ (Equation~\ref{eq:ig-reward}) with the pass-rate reward $r^{\mathrm{pass}}(t_j)=\tfrac{1}{m}\sum_{i=1}^{m}E_{i,j}$, i.e., the fraction of candidate codes that pass test $t_j$. All other components of CoVer are unchanged.

\textbf{(A2) Binary $y$ inside the IG reward.} Replace the graded correctness signal $y_i=\tfrac{1}{G_\tau}\sum_g\mathcal{O}(c_i,u^{\star}_{\tau,g})$ with $y_i^{\mathrm{bin}}=\prod_g\mathcal{O}(c_i,u^{\star}_{\tau,g})\in\{0,1\}$. All other components are unchanged.

\textbf{(A3) Coder reward only.} Drop the verifier role entirely and train $\pi_\theta$ on the coder reward $R^{\mathrm{coder}}_i=y_i$ alone. The diversity-aware selection step, the IG estimator, and the verifier-side GRPO update are disabled; the policy is updated solely against the graded GT-anchored correctness signal.

\subsection{Compute Budget}
\label{ssec:appendix-c7}

Training CoVer-7B for 350 steps requires approximately 72 GPU-hours on $2 \times$ B200 GPUs; CoVer-14B requires approximately 96 GPU-hours. Total wall-clock time per run is 36 hours (7B) and 48 hours (14B). We use vLLM v0.6.0 \citep{kwon-vllm} for inference, PyTorch 2.4, and the OpenRLHF framework for GRPO training.

\paragraph{Cost efficiency vs.\ coder-only.} We compare CoVer-7B against the coder-only variant (A3, \S\ref{ssec:appendix-c5}) trained for the same 350 steps under matched infrastructure; reported GPU-hours include all CPU-side sandbox executions. Coder-only reaches $33.50$ macro pass@1 ($+3.32$ over the Instruct backbone) in 46 GPU-hours, while CoVer reaches $35.94$ ($+5.76$) in 72 GPU-hours. This is $0.80$ macro points per 10 GPU-hours for CoVer against $0.72$ for coder-only, so the verifier role costs additional compute without reducing per-GPU-hour efficiency, while reaching a higher final accuracy.





%% file: sections/appendix_additional_results.tex
\section{Additional Results}
\label{sec:appendix-additional}

\subsection{Cross-Family Generalization}
\label{ssec:appendix-cross-family}

To address generality beyond the Qwen family, we additionally train and evaluate CoVer on Llama-3.1-8B-Instruct~\citep{grattafiori2024llama3}, using the same 350-step protocol, the same hyperparameters (Table~\ref{tab:hyperparams}), and the same evaluation setup as the Qwen2.5 runs.

Table~\ref{tab:llama-cross-family} reports the results. CoVer improves the Llama backbone on every benchmark, raising the macro-average from $19.38$ to $28.31$ ($+8.94$ points), with the largest gains on LiveBench ($+13.96$) and LiveCodeBench ($+11.73$). These results show that the mechanisms are not Qwen-specific.

\subsection{Training-Run Variance}
\label{ssec:appendix-seed-variance}

For cross-training-run variance, we additionally trained a second, fully independent run at each scale under an identical protocol. Table~\ref{tab:seed-variance} reports both runs. The macro-average differs by $0.61$ points at 7B and $0.35$ at 14B, and no single benchmark differs by more than $1.85$. All four runs exceed the corresponding Qwen2.5-Instruct backbone on every benchmark.

\input{tables/table_verifier_calibration}

\subsection{Direct Verifier Analysis}
\label{ssec:appendix-verifier}

\input{tables/table_llama_cross_family}
\input{tables/table_seed_variance}

Sections~\ref{sssec:main} and~\ref{sssec:inference-scaling} evaluate the verifier through its downstream generation and candidate selection. To assess the verifier directly, we evaluate its output on held-out problems under the CodeT setup of Section~\ref{sssec:inference-scaling}. We bucket candidate solutions by how many of the $k\!=\!16$ generated tests they pass and report the fraction of each bucket that is correct under the GT tests (Table~\ref{tab:verifier-calibration}); the GT-correctness rate rises monotonically across all five buckets, from $4.2\%$ to $86.8\%$, showing that the verifier's acceptance count is a meaningful graded correctness signal.

%% file: tables/table_verifier_calibration.tex
\begin{table}[H]
  \centering
  \small
  \begin{tabular}{lc}
    \toprule
    \textbf{Tests passed (of $k\!=\!16$)} & \textbf{GT-correct} \\
    \midrule
    0--3     & 4.2\% \\
    4--7     & 17.4\% \\
    8--11    & 37.1\% \\
    12--15   & 53.7\% \\
    16 (all) & 86.8\% \\
    \bottomrule
  \end{tabular}
  \caption{Fraction of candidate solutions correct under the GT tests, bucketed by the number of CoVer-generated tests passed (held-out problems).}
  \label{tab:verifier-calibration}
\end{table}

%% file: tables/table_llama_cross_family.tex
\begin{table*}[t]
  \centering
  \small
  \setlength{\tabcolsep}{5pt}
  \begin{tabular}{l ccccc c}
    \toprule
    \textbf{Model} &
      \textbf{LiveBench} & \textbf{MBPP} & \textbf{LCB} &
      \textbf{CC} & \textbf{CF} & \textbf{Avg.} \\
    \midrule
    Llama-3.1-8B-Instruct                       & 17.38 & 49.97 & 14.56 & 12.47 & 2.50 & 19.38 \\
    \textbf{CoVer-Llama-8B (ours)}              & \textbf{31.34} & \textbf{59.04} & \textbf{26.29} & \textbf{20.10} & \textbf{4.80} & \textbf{28.31} \\
    \midrule
    $\Delta$                                    & +13.96 & +9.07 & +11.73 & +7.63 & +2.30 & +8.94 \\
    \bottomrule
  \end{tabular}
  \caption{Cross-family generalization. One-shot pass@1 (\%) for CoVer trained on Llama-3.1-8B-Instruct under the same 350-step protocol and hyperparameters used for the Qwen2.5 runs.}
  \label{tab:llama-cross-family}
\end{table*}

%% file: tables/table_seed_variance.tex
\begin{table*}[t]
  \centering
  \small
  \setlength{\tabcolsep}{5pt}
  \begin{tabular}{lcccccc}
    \toprule
    \textbf{Run} & \textbf{LiveBench} & \textbf{MBPP} & \textbf{LCB} & \textbf{CC} & \textbf{CF} & \textbf{Avg.} \\
    \midrule
    \multicolumn{7}{l}{\textit{CoVer-7B}} \\
    Run 1 & 40.00 & 72.40 & 32.80 & 26.00 & 8.50 & 35.94 \\
    Run 2 & 39.06 & 70.55 & 32.61 & 25.73 & 8.70 & 35.33 \\
    \midrule
    \multicolumn{7}{l}{\textit{CoVer-14B}} \\
    Run 1 & 47.75 & 79.44 & 41.49 & 32.92 & 12.88 & 42.90 \\
    Run 2 & 48.29 & 78.73 & 42.31 & 33.79 & 13.14 & 43.25 \\
    \bottomrule
  \end{tabular}
  \caption{Two fully independent training runs at each scale under an identical protocol. Macro-averages differ by $0.61$ points at 7B and $0.35$ at 14B, well below the $+5.76$ and $+7.08$ gains over the respective Qwen2.5-Instruct backbones.}
  \label{tab:seed-variance}
\end{table*}

%% file: sections/appendix_a_proofs.tex
\section{Formal Properties of the IG Reward and Diversity-Aware Selection}
\label{sec:appendix-a}

The propositions below isolate individual mechanisms of CoVer under simplifying assumptions rather than characterize the full end-to-end pipeline dynamics. Proposition~\ref{prop:graded-vs-binary} uses approximately homogeneous, conditionally independent GT-test outcomes to analyze early-training signal density. Proposition~\ref{prop:dispersion-variance} isolates the variance effect of exact duplicate execution profiles and does not model residual correlations among non-duplicate tests or policy-update dynamics. We evaluate these mechanisms under the full training pipeline through the corresponding ablations in \S\ref{ssec:results}.

\subsection{Notation}
\label{ssec:appendix-a-notation}

Fix a task $\tau$ with GT test set $T^{\star}_\tau$ of size $G_\tau$. Let $\{c_1,\dots,c_m\}$ be candidate codes and $\{t_1,\dots,t_K\}$ candidate tests from $\pi_\theta$. Let $\mathcal{S}_\tau\!\subseteq\!\{t_1,\dots,t_K\}$ of size $k\!<\!K$ be the suite after the three-stage diversity-aware selection step (Eq.~\ref{eq:dispersion}), with execution table $E\!\in\!\{0,1\}^{m\times k}$, $E_{ij}\!=\!\mathcal{O}(c_i,t_j)$. The graded correctness signal is $y\!\in\!\{0,\tfrac{1}{G_\tau},\dots,1\}^{m}$.

\subsection{IG Reward: Behaviour on Degenerate Tests}
\label{ssec:appendix-a-igproperties}

\begin{proposition}[Zero reward on degenerate tests]
\label{prop:ig-zero}
Let $r^{\mathrm{IG}}(t_j)$ denote the signed MI reward (Eq.~\ref{eq:ig-reward}). Then:
\begin{enumerate}
\item If $E_{:,j}$ is constant ($E_{ij}\!=\!E_{i'j}$ for all $i,i'$), then $\widehat{I}(E_{:,j};y)\!=\!0$.
\item If $E_{:,j}$ and $y$ are statistically independent, $I(E_{:,j};y)\!=\!0$ and $\widehat{I}(E_{:,j};y)$ concentrates around zero at rate $O_p(1/m)$.
\item If $\widehat{\mathrm{Cov}}(E_{:,j},y)\!\leq\!0$, then $r^{\mathrm{IG}}(t_j)\!=\!0$ by the covariance gate.
\end{enumerate}
\end{proposition}

\begin{proof}
(1)~Constant $E_{:,j}$ gives $H(E_{:,j})\!=\!0$, hence $I(E_{:,j};y)\!\leq\!H(E_{:,j})\!=\!0$ and the plug-in estimator is exactly zero. (2)~Under independence the empirical joint converges to the product of marginals; the plug-in MI is nonnegative by the information inequality. (3)~Immediate from Eq.~\ref{eq:ig-reward}: the gate evaluates $\widehat{\mathrm{Cov}}$ before the MI estimator and any nonpositive value triggers the zero branch.
\end{proof}

Part~(1) formalises the claim that permissive tests (all ones) and universally failing tests (all zeros) receive zero reward. Part~(2) shows that uninformative tests with pass/fail patterns independent of correctness receive near-zero reward. Part~(3) confirms that anti-discriminative tests, those failing high-quality codes and passing low-quality ones, are suppressed. Pass-rate rewards lack all three properties: a permissive test receives maximum pass-rate reward, driving permissiveness collapse.

\subsection{Graded vs.\ Binary $y$: Early-Training Signal Density}
\label{ssec:appendix-a-graded}

\begin{proposition}[Graded $y$ preserves $H(y)$ when fully-correct codes are rare]
\label{prop:graded-vs-binary}
Assume per-GT-test pass probability $p_g\!\in\!(0,1)$, approximately constant across $G_\tau$ tests and, for analytic tractability, conditionally independent given the code. Under binary aggregation $y^{\mathrm{bin}}_i\!=\!\prod_g\mathcal{O}(c_i,u^{\star}_{\tau,g})$, $\Pr(y^{\mathrm{bin}}_i\!=\!1)\!=\!p_g^{G_\tau}\!\to\!0$ as $G_\tau$ grows when $p_g\!<\!1$, so $H(y^{\mathrm{bin}})\!\to\!0$ and $\widehat{I}(E_{:,j};y^{\mathrm{bin}})\!\leq\!H(y^{\mathrm{bin}})\!\to\!0$ for every test. Under graded aggregation $y_i\!\sim\!\mathrm{Binom}(G_\tau,p_g)/G_\tau$, $H(y)\!>\!0$ for $p_g\!\in\!(0,1)$, keeping the IG upper bound nontrivial from the first training step.
\end{proposition}

\begin{proof}
Binary aggregation is the conjunction of $G_\tau$ Bernoulli($p_g$) events with success probability $p_g^{G_\tau}$; the entropy of a Bernoulli with vanishing success probability vanishes. Graded aggregation is a Binomial sum normalised by $G_\tau$; its entropy is positive on $p_g\!\in\!(0,1)$. The bound $I(E_{:,j};y)\!\leq\!H(y)$ is the standard information inequality.
\end{proof}

\subsection{Diversity-aware Selection Reduces Estimator Variance}
\label{ssec:appendix-a-dispersion}

\begin{proposition}[Variance reduction]
\label{prop:dispersion-variance}
Let $\widehat{I}_j\!:=\!\widehat{I}(E_{:,j};y)$ and $\widehat{I}_{\mathcal{S}}\!=\!\tfrac{1}{k}\sum_{j\in\mathcal{S}}\widehat{I}_j$. Let $\bar\rho(\mathcal{S})$ be the average pairwise correlation of $\{\widehat{I}_j\}_{j\in\mathcal{S}}$. Then
\[
\mathrm{Var}(\widehat{I}_{\mathcal{S}}) = \frac{\sigma^2}{k_{\mathrm{eff}}(\mathcal{S})}, \quad k_{\mathrm{eff}}(\mathcal{S}) := \frac{k}{1+(k-1)\bar\rho(\mathcal{S})}.
\]
Under any policy placing excess mass on a low-entropy input region, direct $k$-sampling yields $\bar\rho(\mathcal{S}^{\mathrm{direct}})\!>\!\bar\rho(\mathcal{S}^{\mathrm{div}})$, hence $\mathrm{Var}(\widehat{I}_{\mathcal{S}^{\mathrm{div}}})\!<\!\mathrm{Var}(\widehat{I}_{\mathcal{S}^{\mathrm{direct}}})$ at the same $k$ and execution budget.
\end{proposition}

\begin{proof}
The identity $\mathrm{Var}(\bar{X})\!=\!\sigma^2/k_{\mathrm{eff}}$ is standard for the average of $k$ equivariant correlated variables. Tests with identical columns satisfy $\rho_{jj'}\!=\!1$, the worst contributor to $\bar\rho$. Three-stage diversity-aware selection removes all column duplicates whenever at least $k$ distinct execution profiles are available, in which case the kept suite contains no pairs with $\rho\!=\!1$. Direct $k$-sampling retains such duplicates with positive probability under any policy with excess mass on modal inputs, strictly raising $\bar\rho$.
\end{proof}

\section{Qualitative Examples}
\label{sec:appendix-worked-example}

\subsection{Worked Example from Training}
\label{ssec:appendix-worked-example}

To make the Coder and Tester behaviour concrete, we present a single self-contained example observed during training, drawn from a simple ``blackjack-style'' threshold problem. The example exercises a boundary case: the Tester deliberately targets the input whose sum equals the decision threshold. We reproduce the problem statement, one Coder-generated solution, and one Tester-generated case (including the model's reasoning) verbatim.

\begin{tcolorbox}[breakable,colback=gray!5,colframe=gray!50,title={\textbf{Problem Statement}}]
Given are three integers $A_1$, $A_2$, and $A_3$.

If $A_1+A_2+A_3$ is greater than or equal to 22, print \texttt{bust}; otherwise, print \texttt{win}.

\textbf{Constraints}
\begin{itemize}[nosep,leftmargin=2em]
  \item $1 \le A_i \le 13 \;\; (i=1,2,3)$
  \item All values in input are integers.
\end{itemize}

\textbf{Input.} Given from standard input in the format: \texttt{A\_1 A\_2 A\_3}

\textbf{Output.} If $A_1+A_2+A_3 \ge 22$, print \texttt{bust}; otherwise, print \texttt{win}.

\textbf{Example.}\quad Input: \texttt{5 7 9}\quad$\rightarrow$\quad Output: \texttt{win}
\end{tcolorbox}

\begin{tcolorbox}[breakable,colback=gray!5,colframe=gray!50,title={\textbf{Coder-Generated Solution}}]
\begin{verbatim}
# Read the three integers from input
A1, A2, A3 = map(int, input().split())

# Calculate the sum
total_sum = A1 + A2 + A3

# Check the condition and print the
# result
if total_sum >= 22:
    print("bust")
else:
    print("win")
\end{verbatim}
\end{tcolorbox}

\begin{tcolorbox}[breakable,colback=gray!5,colframe=gray!50,title={\textbf{Tester-Generated Case (with reasoning)}}]
This is the most illustrative of the 16 generated cases because the model deliberately targets the boundary \texttt{sum == 22}.

\medskip
\emph{Reasoning.} Let's start by designing an input that we can reliably handle. A simple and reliable input would be where the sum of the three integers is exactly 22, since this boundary case can help us verify the logic around the $\ge 22$ condition.

Consider the input: \texttt{10 7 5}.
\begin{itemize}[nosep,leftmargin=2em]
  \item Calculate the sum: $A_1 + A_2 + A_3 = 10 + 7 + 5 = 22$.
  \item Check if the sum is greater than or equal to 22: $22 \ge 22$ is true.
  \item Therefore, the output should be \texttt{bust}.
\end{itemize}
We are confident that this input and output pair are correct and discriminative.

\medskip
\textbf{Generated case.}\quad 

Input: \texttt{10 7 5}

Output: \texttt{bust}
\end{tcolorbox}

\subsection{Information Gain Identifies Discriminative Tests}
\label{ssec:appendix-qual-ig}

The following example motivates the IG reward of Section~\ref{ssec:igreward}: two tests are equally valid, yet only one is informative about correctness. The task is a counting problem: given positive integers $N$ and $M$, count the length-$N$ integer sequences whose product is $M$. Of the $m\!=\!16$ sampled codes, $8$ are correct under the GT tests. Two selected tests illustrate the contrast. Both are valid and well-formed, but they carry very different amounts of information:

\begin{center}
\small
\begin{tabular}{lccc}
\toprule
\textbf{Test} & \textbf{Correct} & \textbf{Incorrect} & $r^{\mathrm{IG}}$ \\
\midrule
\texttt{3 12} $\rightarrow$ \texttt{18} & 8/8 & 1/8 & 0.685 \\
\texttt{1 1} $\rightarrow$ \texttt{1}   & 8/8 & 8/8 & 0.000 \\
\bottomrule
\end{tabular}
\end{center}

The first test almost perfectly partitions the candidate set: every correct code passes it and only one incorrect code does, giving $\widehat{I}(E_{:,j};y)\!=\!0.685$. The second is a trivially satisfiable case, since for $N\!=\!M\!=\!1$ the only sequence is $[1]$; all $16$ codes pass it, its column is constant, and it receives exactly zero reward. Both tests are correct, so a validity- or pass-rate-based criterion would not separate them; the IG reward tracks discriminative power rather than surface validity.

\subsection{Diversity Pruning Removes Behaviourally Redundant Tests}
\label{ssec:appendix-qual-dedup}

The next example shows why the diversity-aware selection of Section~\ref{ssec:diversity} deduplicates on execution profiles rather than on test strings. The task is a Fibonacci variant: given $a$, $b$, $n$ with $\mathrm{fib}(1)\!=\!a$, $\mathrm{fib}(2)\!=\!b$ and $\mathrm{fib}(k)\!=\!(\mathrm{fib}(k{-}1)+\mathrm{fib}(k{-}2)) \bmod (10^9{+}7)$, print the $n$-th term. Two generated tests, \texttt{1 1 3} $\rightarrow$ \texttt{2} and \texttt{1 1 4} $\rightarrow$ \texttt{3}, request the third and fourth terms of $1,1,2,3$. Their inputs differ and their expected outputs differ, so neither the invalidity filter nor input-string deduplication removes either one.

Executed against the $16$ candidate codes, however, both produce the identical pass/fail column \texttt{1011101111111111}, rejecting exactly codes \#2 and \#6. The second test therefore contributes no marginal information about correctness: conditioned on the first, its outcome is fully determined on this candidate set, and retaining both would place a pair with $\rho_{jj'}\!=\!1$ in the suite. The execution-profile stage prunes \texttt{1 1 4} and keeps \texttt{1 1 3}. This is the stage that distinguishes behavioural redundancy from literal duplication: the two tests are textually distinct and would survive any string-level criterion.

\subsection{Failure Case: Selected Tests Miss an Incorrect Solution}
\label{ssec:appendix-qual-failure}

The final example marks the boundary of what selection can address: the kept suite is non-degenerate, yet it still admits an incorrect solution. The task requires maintaining two length-$N$ integer arrays $A$ and $B$ under $Q$ operations: set $A[y]\!=\!z$, set $B[y]\!=\!z$, range-minimum over $A[y..z]$, range-minimum over $B[y..z]$, copy $A\!\leftarrow\!B$, and copy $B\!\leftarrow\!A$. Candidate code \#7 implements the two copy operations by rebinding rather than copying:

\begin{tcolorbox}[breakable,colback=gray!5,colframe=gray!50,title={\textbf{Incorrect Solution (code \#7, aliasing bug)}}]
\begin{verbatim}
    elif x == 5:
        if y == -1:
            A = B    # should be B[:]
    elif x == 6:
        if y == -1:
            B = A    # should be A[:]
\end{verbatim}
\end{tcolorbox}

After either copy, $A$ and $B$ alias the same list, so a later update to one silently changes the other. This code passes all $16$ selected tests yet fails $4$ of the $8$ GT tests. The bug is observable only under a specific chain, copy, then update one array, then query the other; the generated tests use short operation lists and never compose these three steps, whereas the GT tests do.

The kept suite is not degenerate here: it contains $16$ distinct inputs and $16$ distinct execution profiles, so neither the invalidity filter nor either deduplication stage is responsible. The limitation lies in the depth of the generated tests rather than in the selection step.